\documentclass[a4paper]{article}

\usepackage[affil-it]{authblk}

\usepackage{amsmath}
\usepackage{amsthm}
\usepackage{amssymb}
\usepackage{amsfonts}
\usepackage{euscript}

\usepackage{float}
\usepackage{graphicx}

\usepackage{multirow}
\usepackage{subfig}
\usepackage{tikz}
\newcommand{\clust}{\textsf{clust}}
\newcommand{\Paths}{\textsf{Paths}}

\newcommand{\leqs}{\leqslant}
\newcommand{\geqs}{\geqslant}

\newenvironment{eroman}{
\begin{enumerate}%
\def\theenumi{\rm\roman{enumi}}%
}{\end{enumerate}}

\newlength{\alginputlength}
\newcommand{\ignore}[1]{}

\newlength{\savefboxsep}
\newcommand{\st}{\,\mid\,}

\newcommand{\lw}{\longrightarrow}

\newcounter{saveenumi}
\newcommand{\startbreakenumi}
     {\vspace{3mm}
     \setcounter{saveenumi}{\theenumi}
      \end{enumerate}}
\newcommand{\stopbreakenumi}
     {\vspace{3mm}
     \begin{enumerate}
       \setcounter{enumi}{\thesaveenumi}}
\newcounter{saveenumii}
\newcommand{\startbreakenumii}
     {\setcounter{saveenumii}{\arabic{enumii}}
      \end{enumerate}}
\newcommand{\stopbreakenumii}
     {\begin{enumerate}
       \setcounter{enumii}{\thesaveenumii}}

\newtheorem{THM}{Theorem}[section]
\newtheorem{de}[THM]{Definition}

\newtheorem{exm}[THM]{Example}
\newenvironment{EX}{\begin{exm}\rm}{\hspace*{\fill}\raisebox{0.35em}{\framebox[0.4em]{\rule{0em}{0.12em}}}\end{exm}}

\newcommand{\calc}{{\cal C}}

\newlength{\vdl}

\newcommand{\bdef}{\left\{ \begin{array}{ll}}
\newcommand{\fdef}{\end{array} \right.}

\font\bbd=msbm10      % "blackboard bold"
\newcommand{\rr}{\hbox{\bbd R}}
\newcommand{\nn}{\hbox{\bbd N}}

\newcommand{\pp}{\hbox{\bbd P}}

\newlength{\vhx}

\newcommand{\eud}{{\cal D}}

\newcommand{\euu}{{\cal U}}

\date{}

\colorlet{relation@colour}[rgb]{blue}
\colorlet{entity@colour}[rgb]{red}

\begin{document}

\title{Data ultrametricity and clusterability}

\author{Dan Simovici%
  \thanks{\texttt{Dan.Simovici@umb.edu}; Corresponding author}}
\affil{Department of Computer Science,\\ University of Massachusetts Boston, \\ Boston, USA}

\author{Kaixun Hua%
  \thanks{\texttt{kingsley@cs.umb.edu}}}
\affil{Department of Computer Science,\\ University of Massachusetts Boston, \\ Boston, USA}

\maketitle

\begin{abstract}
The increasing needs of clustering massive datasets and the high cost
of running clustering algorithms poses difficult problems for users.
In this context it is important to determine if a data set is
clusterable, that is, it may be partitioned efficiently into
well-differentiated groups containing similar objects.
We approach data clusterability from an ultrametric-based perspective.
A novel approach to determine the ultrametricity of a dataset is
proposed via a special type of matrix product, which allows us to
evaluate the clusterability of the dataset. Furthermore, we show that
by applying our technique to a dissimilarity space will generate the
sub-dominant ultrametric of the dissimilarity.
\end{abstract}

\section{Introduction}\label{sec:intro}
Clustering is the prototypical unsupervised learning activity which
consists in identifying cohesive and well-differentiated groups of
records in data. A data set is clusterable if such groups exist;
however, due to the variety in data distributions and the inadequate
formalization of certain basic notions of clustering, determining data
clusterability before applying specific clustering algorithms is a
difficult task.

Evaluating data clusterability before the application of clustering
algorithms can be very helpful because clustering algorithms are
expensive. However, many such evaluations are impractical because they
are NP-hard, as shown in~\cite{MASB}. Other
notions define data as clusterable when
the minimum between-cluster separation is
greater than the maximum intra-cluster
distance~\cite{ESK}, or when each element is
closer to all elements in its cluster than to all other
data~\cite{BBV}.

Several approaches exist in assessing data clusterability.
The main hypothesis of~\cite{AAB} is that clusterability can
be inferred from an one-dimensional view of pairwise distances between
objects.  Namely, clusterability is linked to the multimodality of the
histogram of inter-object dissimilarities.  The basic assumption is
that ``the presence of multiple modes in the set of pairwise
dissimilarities indicates that the original data is clusterable.''
Multimodality is evaluated using the {\em Dip} and {\em Silverman}
statistical multimodality tests, an approach that is computationally efficient.

Alternative approaches to data clusterability are linked to
the feasibility of producing a clustering;
a corollary of this assumption is that ``data that are hard to cluster
do not have a meaningful clustering
structure''~\cite{DANLS}.  Other approaches to
clusterability are identified based on clustering quality measures,
and on loss function optimization~\cite{MASB,MASB08,ASLO10,SBD,BBV,BHE}.

We propose a novel approach that relates data clusterability to the
extent to which the dissimilarity defined on the data set relate to a
special ultrametric defined on the set.

The paper is structured as follows. In Section~\ref{sec:disu} we
introduce dissimilarities and an ultrametrics that play a central
role in our definition of clusterability.  A special matrix
product on matrices with non-negative elements that allow an
efficient computation of the subdominant ultrametric is
introduced. In Section~\ref{sec:mc} a measure of clusterability
that is based on the iterative properties of the dissimilarity
matrix is defined. We provide experimental evidence on the
effectiveness of the proposed measure through several experiments
on small artificial data sets in Section~\ref{sec:ev}. Finally, we
present our conclusions and future plans in Section~\ref{sec:cfw}.

\section{Dissimilarities, Ultrametrics, and Matrices}\label{sec:disu}

A dissimilarity on a set $S$
is a mapping $d: S\times S \lw \rr$ such that
\begin{eroman}
\item
$d(x,y) \geqs 0$ and $d(x,y)=0$ if and only if $x=y$;
\item
$d(x,y) = d(y,x)$;
\end{eroman}

A dissimilarity on $S$ that satisfies the triangular inequality
\[
d(x,y) \leqs d(x,z) + d(z,y)
\]
for every $x,y,z\in S$ is a metric.  If, instead, the stronger inequality
\[
d(x,y) \leqs \max \{d(x,z),d(z,y)\}
\]
is satisfied, $d$ is said to be an {\em ultrametric} and the pair
$(S,d)$ is an {\em ultrametric space}.

A {\em closed sphere} in $(S,d)$ is a set $B[x,r]$
defined by
\[
B[x,r] = \{y \in S \st d(x,y) \leqs r\}.
\]
When $(S,d)$ is an ultrametric space two spheres having the same
radius $r$ in $(S,d)$ are either disjoint or coincide~\cite{DSCD}.
Therefore, the collection of closed spheres of radius $r$ in $S$,
$\calc_r = \{B[x,r]\st r\in S\}$ is a partition of $S$; we refer to
this partition as an {\em $r$-spheric clustering} of $(S,d)$.

In an ultrametric space $(S,d)$ every triangle is isosceles.  Indeed,
let $T = (x,y,z)$ be a triplet of points in $S$ and let $d(x,y)$ be the
least distance between the points of $T$.  Since
$d(x,z) \leqs \max \{d(x,y),d(y,z)\} = d(y,z)$ and
$d(y,z) \leqs \max \{d(y,x),d(x,z)\} = d(x,z)$, it follows that
$d(x,z) = d(y,z)$, so $T$ is isosceles; the two longest sides of this
triangle are equal.

It is interesting to note that every  $r$-spheric clustering in an
ultrametric space is a perfect clustering~\cite{ABBL}.  This means that
all of its in-cluster distances are smaller than all of its between-cluster
distances.  Indeed, if $x,y$ belong to the same cluster $B[u,r]$ then
$d(x,y) \leqs r$.  If $x \in B[u,r]$ and $y \in B[v,r]$, where $B[u,r]
\cap B[v,r] = \emptyset$, then $d(v,x) > r$, $d(y,v) \leqs r$ and this
implies $d(x,y) = d(x,v) > r$ because the triangle $(x,y,v)$ is isosceles and
$d(y,v)$ is not the longest side of this triangle.

\begin{EX}\label{exm:nov1118a}
Let $S = \{x_i \st 1\leqs i \leqs 8\}$ and let $(S,d)$ be the
ultrametric space, where the ultrametric $d$ is defined by the
following table:
\[
\begin{array}{|c||cccccccc|}\hline
d(x_i,x_j) & x_1 & x_2 & x_3 & x_4 & x_5 & x_6 & x_7 & x_8\\ \hline
x_1 & 0& 4& 4& 10& 10& 16& 16& 16\\
x_2 & 4& 0& 4& 10& 10& 16& 16& 16\\
x_3 & 4& 4& 0& 10& 10& 16& 16& 16\\
x_4 & 10& 10& 10& 0& 6& 16& 16& 16\\
x_5 & 10& 10& 10& 6& 0& 16& 16& 16\\
x_6 & 16& 16& 16& 16& 16& 0& 4& 4\\
x_7 & 16& 16& 16& 16& 16& 4& 0& 4\\
x_8 & 16& 16& 16& 16& 16& 4& 4& 0\\ \hline
\end{array}
\]
The closed spheres of this spaces are:
\begin{eqnarray*}
B[x_i,r] &=& \begin{cases}
           \{x_i\} & \mbox{ for } r<4,\\
           \{x_1,x_2,x_3\} & \mbox{ for } 4 \leqs r < 10,\\
           \{x_1,x_2,x_3,x_4,x_5\} & \mbox{ for } 10 \leqs r < 16,\\
           S & \mbox{ for } r = 16,
           \end{cases}\\
         & & \mbox{ for }1\leqs i \leqs 3,\\
B[x_i,r] &=& \begin{cases}
             \{x_i\} & \mbox{ for } r < 6,\\
             \{x_4,x_6\} & \mbox{ for } 6 \leqs r < 16,\\
             S & \mbox{ for } r=16,
             \end{cases}\\
         & & \mbox{for }4\leqs i \leqs 5,\\
B[x_i,r] &=& \begin{cases}
             \{x_i\} & \mbox{ for }r < 4,\\
             \{x_6,x_7,x_8\} & \mbox{ for } 4\leqs r < 16,\\
             S & \mbox{ for } r = 16,
             \end{cases}\\
         & & \mbox{ for }6\leqs i \leqs 8.
\end{eqnarray*}
\end{EX}
Based on the properties of spheric clusterings mentioned above
meaningful such clusterings can be produced in linear time in the
number of objects.  For the ultrametric space mentioned in
Example~\ref{exm:nov1118a}, the closed spheres of radius $6$ produce
the clustering
\[
\{x_1,x_2,x_3\},\{x_4,x_5,\},\{x_6,x_7,x_8\}.
\]
If a dissimilarity defined on a data set
is close to an ultrametric it is
natural to assume that the data set is clusterable.  We assess the
closeness between a dissimilarity $d$ and a special ultrametric known as
the {\em subdominant ultrametric} of $d$ using a matrix approach.

Let $S$ be a set. Define a partial order ``$\leqs$'' on the set of
definite dissimilarities $\eud_S$ by $d \leqs d'$ if $d(x,y)\leqs d'(x,y)$
for every $x,y\in S$.  It is easy to
verify that $(\eud_S,\leqs)$ is a poset.

The set $\euu_S$ of ultrametrics on $S$ is a subset of $\eud_S$.
\begin{THM}\label{thm:apr2818z}
Let $\{d_i \in \euu_S \st i\in I\}$ be a collection of ultrametrics
on the set $S$.  Then, the mapping $d:S \times S \lw \rr_{\geqs 0}$
defined as
\[
d(x,y) = \sup \{d_i(x,y) \st i\in I\}
\]
is an ultrametric on $S$.
\end{THM}

\begin{proof}
We need to verify only that $d(x,y)$ satisfies the ultrametric
inequality $d(x,y) \leqs \max \{d(x,z),d(z,y)\}$ for $x,y,z \in S$.
Since each mapping $d_i$ is an ultrametric,  for $x,y,z \in S$ we have
\begin{eqnarray*}
d_i(x,y) &\leqs& \max \{d_i(x,z),d_i(z,y)\} \\
         &\leqs& \max \{d(x,z),d(z,y)\}
\end{eqnarray*}
for every $i \in I$.  Therefore,
\begin{eqnarray*}
d(x,y) &=& \sup \{d_i(x,y) \st i\in I\}\\
       &\leqs & \max \{d(x,z),d(z,y)\},
\end{eqnarray*}
hence $d$ is  an ultrametric on $S$.
\end{proof}

\begin{THM}\label{thm:dec1504a}
Let $d$ be a dissimilarity on a set $S$ and let $U_d$ be the set of
ultrametrics $U_d = \{e\in \euu_S \st e\leqs d\}$.  The set $U_d$ has
a largest element in the poset $(\euu_S,\leqs)$.
\end{THM}

\begin{proof}
The set $U_d$ is nonempty because the zero dissimilarity
$d_0$ given by $d_0(x,y) = 0$ for every $x,y\in S$ is an ultrametric
and $d_0 \leqs d$.

Since the set $\{e(x,y) \st e\in U_d\}$ has $d(x,y)$ as an upper
bound, it is possible to define the mapping $e_1 :S^2 \lw \rr_{\geq
0}$ as $e_1(x,y) = \sup \{e(x,y)\st e \in U_d\}$
for $x,y\in S$.  It is clear that $e\leqs e_1$ for every ultrametric
$e$.  We claim that $e_1$ is an ultrametric on $S$.

We prove only that $e_1$ satisfies the ultrametric inequality.
Suppose that there exist $x,y,z\in S$ such that $e_1$ violates the
ultrametric inequality; that is,
\[
\max\{e_1(x,z),e_1(z,y)\} < e_1(x,y).
\]
This is equivalent to
\begin{eqnarray*}
\lefteqn{\sup \{e(x,y)\st e\in U_d\}}\\
&>& \max \{ \sup\{e(x,z)\st e\in U_d\},\\
& &  \sup\{e(z,y)\st e\in U_d\}\}.
\end{eqnarray*}
Thus, there exists $\hat{e}\in U_d$ such that
\[
\hat{e}(x,y)  >  \sup\{e(x,z)\st e\in U_d\}
\]
and
\[
\hat{e}(x,y)  >  \sup\{e(z,y)\st e\in U_d\}.
\]

In particular, $\hat{e}(x,y) > \hat{e}(x,z)$ and $\hat{e}(x,y) >
\hat{e}(z,y)$, which contradicts the fact that $\hat{e}$ is an
ultrametric.
\end{proof}

The ultrametric defined by Theorem~\ref{thm:dec1504a} is known as
the {\em maximal subdominant ultrametric for the dissimilarity
$d$}.\index{maximal subdominant ultrametric for a dissimilarity}

The situation is not symmetric with respect to the infimum of a set
of ultrametrics because, in general, the infimum of a set of
ultrametrics is not necessarily an ultrametric.

Let $\pp$ be the set
\[
\pp = \{x \st x\in \rr, x\geqs 0\} \cup \{\infty\}.
\]
The usual operations defined on $\rr$ can be extended to $\pp$ by
defining
\[
x + \infty = \infty + x = \infty, x\cdot \infty = \infty \cdot x = \infty
\]
for $x \geqs 0$.

Let $\pp^{m\times n}$ be the set of $m \times n$ matrices over $\pp$.
If $A,B \in \pp^{m\times n}$ we have $A \leqs B$ if $a_{ij} \leqs
b_{ij}$ that is, if $a_{ij} \geqs b_{ij}$ for $1 \leqs i \leqs m$ and
$1\leqs j \leqs n$.

If $A\in \pp^{m\times n}$ and $B \in \pp^{n \times p}$ the
matrix product $C = AB \in \pp^{m \times p}$ is defined as:
\[
c_{ij} = \min \{\max \{a_{ik},b_{kj}\}\st 1 \leqs k \leqs n\},
\]
for $1 \leqs i \leqs m$ and $1\leqs j \leqs p$.

If $E_n \in \pp^{n\times n}$ is the matrix defined by
\[
(E_n)_{ij} = \begin{cases}
             0 & \mbox{if } i=j,\\
             \infty & \mbox{otherwise},
             \end{cases}
\]
that is the matrix whose main diagonal elements are $0$ and the other
elements equal $\infty$, then $A E_n = A$ for every $A \in
\pp^{m \times n}$ and $E_n A = A$ for every $A\in \pp^{n\times p}$.

The matrix multiplication defined above is associative, hence
$\pp^{n\times n}$ is a semigroup with the identity $E_n$.  The powers of
$A$ are inductively defined as
\begin{eqnarray*}
A^0 &=& E_n,\\
A^{n+1} &=& A^n A,
\end{eqnarray*}
for $n \in \nn$.

For $A,B \in \pp^{m \times n}$ we define $A \leqs B$ as $a_{ij}
\leqs B_{ij}$ for $1\leqs i \leqs m$ and $1 \leqs j \leqs n$.
Note that if $A \in \pp^{n\times n}$, then $A \leqs E_n$.  It is
immediate that for $A,B \in \pp^{m \times n}$ and $C \in \pp^{n
\times p}$, then $A\leqs B$ implies $AC \leqs BC$; similarly, if
$C \in \pp^{p\times m}$ and $CA \leqs CB$.

Let $L(A)$ be the finite set of elements in $\pp$ that occur in the
matrix $A \in \pp^{n\times n}$.  Since he entries of any power $A^n$
of $A$ are also included in $L(A)$, the sequence
$A,A^2,\ldots,A^n,\ldots$ is ultimately periodic because it contains a
finite number of distinct matrices.

Let $k(A)$ be the least integer $k$ such that $A^k = A^{k+d}$ for some
$d > 0$.  The sequence of powers of $A$ has the form
\[
\begin{array}{l}
A,A^2,\ldots,A^{k(A)-1},A^{k(A)},\ldots,\\
\quad A^{k(A)+d-1}, A^{k(A)},\ldots,A^{k(A)+d-1},\ldots,
\end{array}
\]
where $d$ is the least integer such that $A^{k(A)} = A^{k(A)+d}$.
This integer is denoted by $d(A)$.

The set $\{A^{k(A)},\ldots,A^{k(A)+d-1}\}$
is a cyclic group with respect to the multiplication.

If $(S,d)$ is a dissimilarity space, where $S = \{x_1,\ldots,x_n\}$,
the matrix of this space is the matrix $A \in \pp^{n\times n}$
defined by $a_{ij} = d(x_i,x_j)$ for $1 \leqs i,j \leqs n$.  Clearly,
$A$ is a symmetric matrix and all its diagonal elements are $0$,
that is, $A \leqs E_n$.

If, in addition, we have $a_{ij} \leqs a_{ik} + a_{kj}$ for $1 \leqs
i,j,k \leqs n$, then $A$ is a {\em metric matrix}.  If this condition
is replaced by the stronger condition $a_{ij} \leqs \max \{a_{ik} +
a_{kj}\}$ for $1 \leqs i,j,k \leqs n$, then $A$ is {\em ultrametric
matrix}.  Thus, for an ultrametric matrix we have $a_{ij} \leqs \min
\{\max \{a_{ik} + a_{kj}\} \st 1\leqs k \leqs n\}$. This amounts to
$A^2 \leqs A$.

\ignore{
\begin{EX}
Let
\[
A = \begin{pmatrix}
    0 & a & b \\
    a & 0 & c\\
    b & c & 0
    \end{pmatrix}
\]
be a symmetric $3\times 3$ matrix with non-negative entries.  It is
easy to see that
\[
A^2 = \begin{pmatrix}
      0 & \min\{a,\max \{b,c\}\} & \min \{b,\max\{a,c\}\}\\
      \min\{a,\max \{b,c\}\}  & 0 & \min \{c,\max\{a,b\}\}\\
      \min \{b,\max\{a,c\}\}  & \min \{c,\max\{a,b\}\} & 0
      \end{pmatrix}.
\]
Note that $A$ is an ultrametric matrix, (say $a \leqs
\max \{b,c\}$) if and only if $A = A^2$.
\end{EX}
}

\begin{THM}\label{thm:may0618z}
If $A \in \pp^{n\times n}$ is a dissimilarity matrix there exists $m\in \nn$ such that
\[
\cdots = A^{m+1} = A^m \leqs \cdots \leqs A^2 \leqs A \leqs  E_n
\]
and $A^m$ is an ultrametric matrix.
\end{THM}

\begin{proof}
Since $A \leqs E_n$, the existence of the number $m$ with the property mentioned
in the theorem is immediate since there exists only a finite number
of $n \times n$ matrices whose elements belong to $L(A)$.
Since $A^m = A^{2m}$, it follows that $A^m$ is an ultrametric matrix.
\end{proof}

For a matrix $A \in \pp^{n\times n}$ let $m(A)$ be the least number
$m$ such that $A^m = A^{m+1}$.  We refer to $m(A)$ as the
{\em stabilization power} of the matrix $A$.  The matrix
$A^{m(A)}$ is denoted by $A^*$.

The previous considerations suggest defining the {\em ultrametricity}
of a matrix $A \in \pp^{n\times n}$ with $A \leqs E_n$ as $u(A) =
\frac{n}{m(A)}$.  Since $m(A) \leqs n$, it follows that $u(A) \geqs
1$.  If $m(A) = 1$, $A$ is ultrametric itself and $u(A) = n$.

\begin{THM}\label{thm:may0618y}
Let $(S,d)$ be a dissimilarity space, where $S = \{x_1,\ldots,x_n\}$
having the dissimilarity matrix $A \in \pp^{n\times n}$.  If $m$ is the least
number such that $A^m = A^{m+1}$, then the mapping $\delta : S \times
S \lw \pp$ defined by $\delta(x_i,x_j) = (A^m)_{ij}$ is the
subdominant ultrametric for the dissimilarity $d$.
\end{THM}

\begin{proof}
As we observed, $A^m$ is an ultrametric matrix, so $\delta$ is an
ultrametric on $S$.  Since $A^m \leqs A$, it follows that $d(x_i,x_j)
\geqs \delta(x_i,x_j)$ for all $x_i,x_j \in S$.

Suppose that $C \in \pp^{n\times  n}$ is an ultrametric matrix such
that $A \leqs
C$, which implies $A^m \leqs C^m \leqs C$.  Thus, $A^m$ dominates any
ultrametric that is dominated by $d$.  Consequently, the dissimilarity
defined by $A^m$ is the subdominant ultrametric for $d$.
\end{proof}

The subdominant ultrametric of a dissimilarity is usually studied in
the framework of weighted graphs~\cite{LECL}.

A {\em weighted graph} is a triple $(V,E,w)$, where $V$ is the set of vertices
of $G$, $E$ is a set of two-element subsets of $V$ called edges.
and $w: E \lw \pp$ is the weight of the edges.  If $e\in E$,
then $e = \{u,v\}$, where $u,v$ are distinct vertices in $V$.
The weight is extended to all 2-elements subsets of $V$ as
\[
w(\{v_i,v_j\}) = \begin{cases}
                   w(\{v_i,v_j\}) & \mbox{if } \{v_i,v_j\} \in E,\\
                   \infty & \mbox{otherwise}.
                   \end{cases}
\]
A {\em path of length $n$} in a weighted graph is a sequence
\[
\wp = (v_0,v_1,,v_2,\ldots,v_{n-1},v_n),
\]
where $\{v_i,v_{i+1}\}\in E$ for $0 \leqs n \leqs n-1$.

The set of paths of length $n$ in the graph $G$ is denoted as
$\Paths^n(G)$.  The set of paths of length $n$ that join the vertex
$v_i$ to the vertex $v_j$ is denoted by $\Paths^{n}_{ij}$.
The set of all paths is
\[
\Paths(G) = \bigcup_{n\geqs 1} \Paths^n(G).
\]

For a weighted graph $G = (V,E,w)$, the extension of the weight
function $w$ to $\Paths^n(G)$ is the function $M : \Paths(G) \lw \pp$
defined as
\[
M(\wp) = \max \{w(v_{i-1},v_i)\st 1 \leqs i \leqs n\},
\]
where $\wp = (v_0,v_1,\ldots,v_n)$.  
Thus, if $\wp' = \wp e$, we have $M(\wp') = \max \{M(\wp),w(e)\}$.

If $G = (V,E,w)$ is a weighted graph, its {\em incidence matrix} is
the matrix $A_G \in \pp^{n \times n}$, where $n = |V|$, defined by
$(A_G)_{ij} = w(v_i,v_j)$ for $1\leqs i,j \leqs n$.

Let $P_{ij}^{(\ell)}$ be the set of paths of length $\ell$ that
join the vertex $v_i$ to the vertex $v_j$. Note that
\begin{eqnarray*}
P_{ij}^{(\ell+1)} &=& \{(v_i,\ldots,v_k,v_j) \st \\
                  & & \wp = (v_i,\ldots,v_k)\in P_{ik}^{(\ell)} \mbox{
		  and } \\
                  & & v_j \mbox{ does not occur in } \wp\}.
\end{eqnarray*}

Define $a_{ij}^{(\ell)} = \min \{M(\wp) \st \wp \in P_{ij}^{(\ell)} \}$.
The powers of the incidence matrix of the graph are given by
\begin{eqnarray*}
a_{ik}^{(\ell+1)} &=& \min \{M(\wp') \st \wp' \in P_{ik}^{(\ell+1)} \}\\
                  &=& \min \{\max \{M(\wp),w(e)\} \st \\
                  & & \wp' = (v_i,\ldots,v_j,v_k) \mbox{ and } \\
                  & & \wp \in P_{ij}^{(\ell)}, e = (v_j,v_k)\in E\}\\
                  &=& \min_{j} \{\max \{a_{ij}^{\ell},w(e)\} \st e = (v_j,v_k)\}.
\end{eqnarray*}
Thus, we have
\[
(A_{G}^{\ell})_{ij} = \min \{M(\wp) \st \wp \in P_{ij}^{\ell}\}
\]
for $1 \leqs i, j \leqs n$.

\section{A Measure of Clusterability}\label{sec:mc}
We conjecture that a dissimilarity space $(D,d)$ is more clusterable
if the dissimilarity is closer to an ultrametric, hence if $m(A_D)$ is
small.  Thus, it is natural to define the {\em clusterability of a
data set} $D$ as the number $\clust(D) = \frac{n}{m(A_D)}$ where $n =
|D|$, $A_D$ is the dissimilarity matrix of $D$ and $m(A_D)$ is the
stabilization power of $A_D$.  The lower the stabilization power, the
closer $A$ is to an ultrametric matrix, and thus, the higher the
clusterability of the data set.

\begin{table*}[ht]
\caption{\label{tbl:nov2718a}
All clusterable datasets have values greater than 5 for their
clusterability; all non-clusterable datasets have values no larger
than $5$.}
\centering
\begin{tabular}{|l|l|l|l|l|l|}\hline
Dataset                & n    &  Dip   &   Silv. & $m(A_D)$ & $\clust(D)$ \\ \hline
iris            & 150  & 0.0000 &  0.0000 & 14     & 10.7\\
swiss           & 47   & 0.0000 &  0.0000 & 6      & 7.8\\
faithful        & 272  & 0.0000 &  0.0000 & 31     & 8.7\\
rivers          & 141  & 0.2772 &  0.0000 & 22     & 6.4\\
trees           & 31   & 0.3460 &  0.3235 & 7      & 4.4\\
USAJudgeRatings & 43   & 0.9938 &  0.7451 & 10     & 4.3\\
USArrests       & 50   & 0.9394 &  0.1897 & 15     & 3.3\\
attitude        & 30   & 0.9040 &  0.9449 & 6      & 5\\
cars            & 50   & 0.6604 &  0.9931 & 15     & 3.3\\ \hline
\end{tabular}
\end{table*}

Our hypothesis is supported by previous results obtained
in~\cite{AAB}, where the clusterability of 9 databases were
statistically examined using the Dip and Silverman tests of
unimodality.  The approach used in~\cite{AAB} starts with the
hypothesis that the presence of multiple modes in the uni-dimensional
set of pairwise distances indicates that the original data set is
clusterable.  Multimodality is assessed using the tests mentioned
above.  The time required by this evaluation is quadratic in the
number of objects.

The first four data sets, {\em iris}, {\em swiss}, {\em faithful} and
{\em rivers} were deemed to be clusterable; the last five were
evaluated as not clusterable.  Tests published in~\cite{AABR} have
produced low $p$-values for the first four datasets, which is an
indication of clusterability.  The last five data sets, {\em
USArrests}, {\em attitude}, {\em cars}, and {\em trees} produce much
larger $p$-values, which show a lack of clusterability.
Table~\ref{tbl:nov2718a} shows that all data sets deemed clusterable
by the unimodality statistical test have values of the clusterability
index that exceed $5$.

In our approach clusterability of a data set $D$ is expressed
primarily through the ``stabilization power'' $m(A_D)$ of the
dissimilarity matrix $A_D$; in addition, the histogram of the
dissimilarity values is less differentiated when the data is not
clusterable.

\section{Experimental Evidence on Small Artificial Data
Sets}\label{sec:ev} Another series of experiments involved a
series of small datasets having the same number of points in
$\rr^2$ arranged in lattices.  The points have integer coordinates
and the distance between points is the Manhattan distance.

\begin{figure*}
\centering
\begin{tabular}{cc}
\multicolumn{2}{c}{
\includegraphics[width=0.27\textwidth]{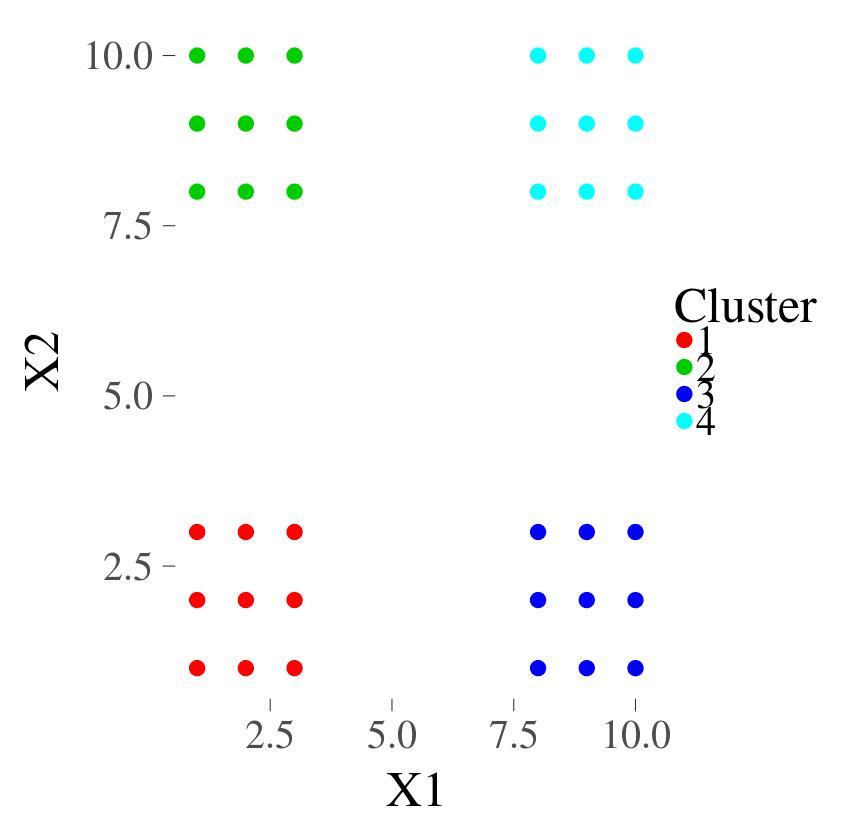}}\\
\multicolumn{2}{c}{Original dataset}\\
\includegraphics[width=0.27\textwidth]{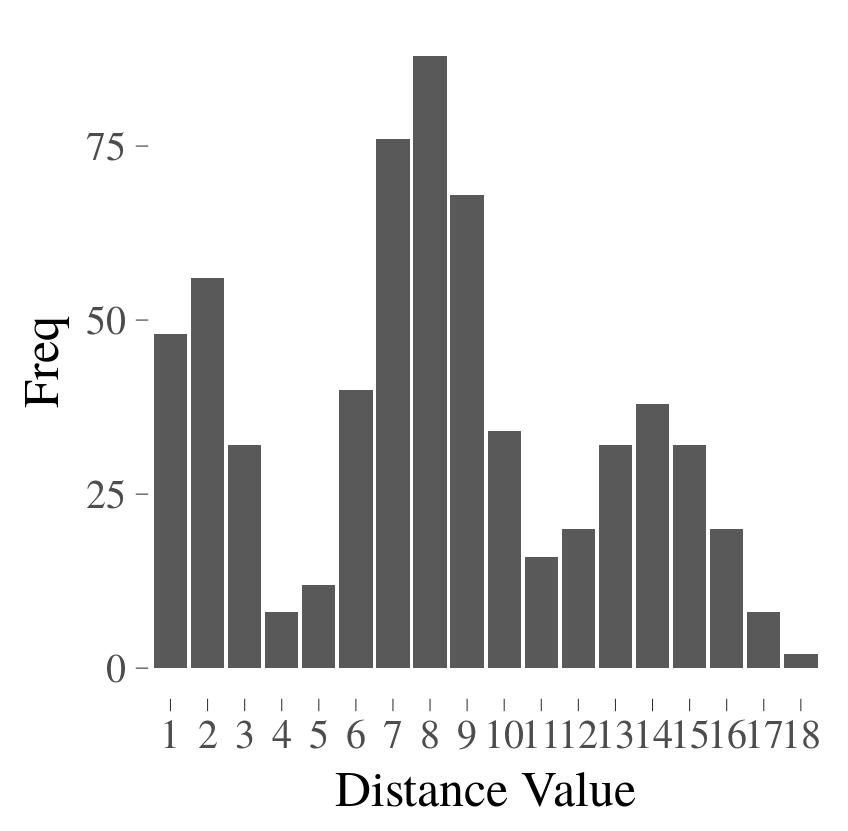}
&
\includegraphics[width=0.27\textwidth]{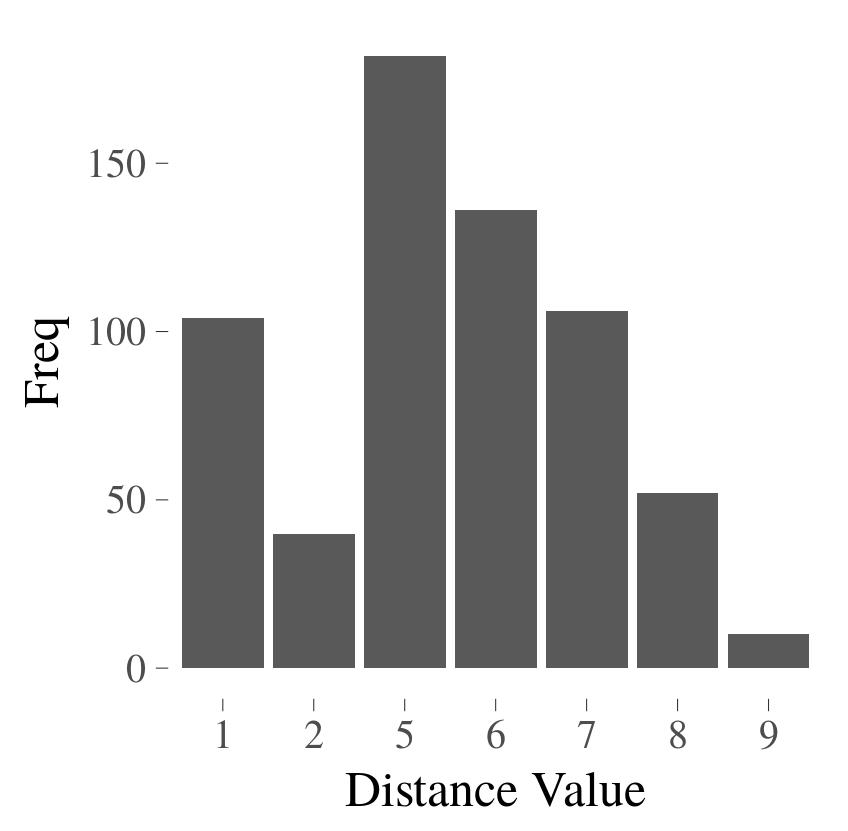}\\
Histogram of original
&
Histogram after one multiplication
\\
\includegraphics[width=0.27\textwidth]{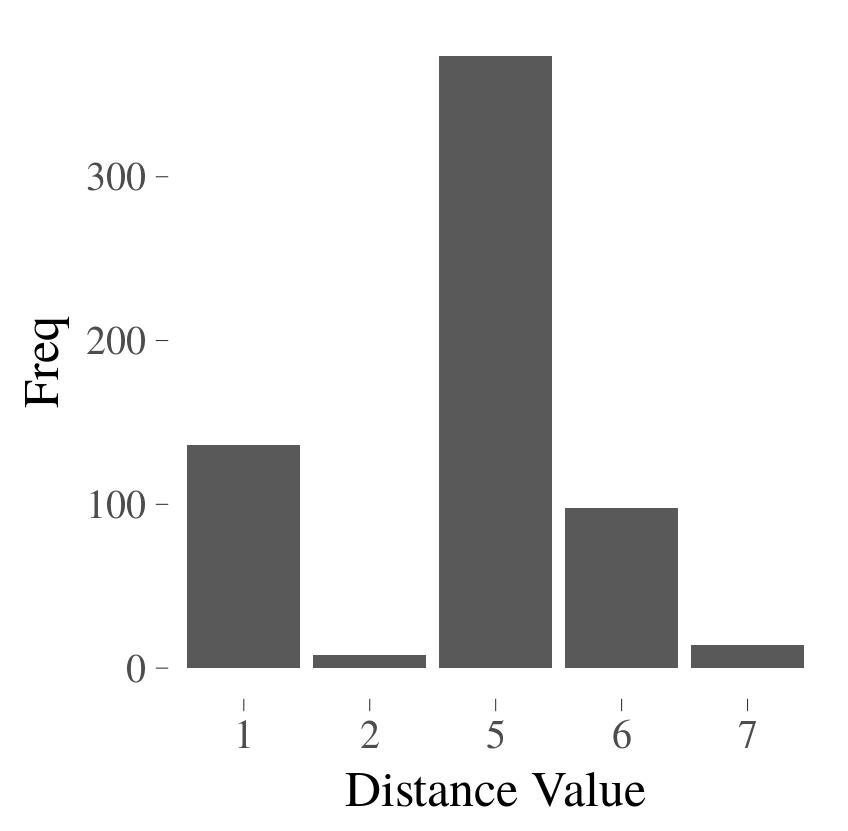}
&
\includegraphics[width=0.27\textwidth]{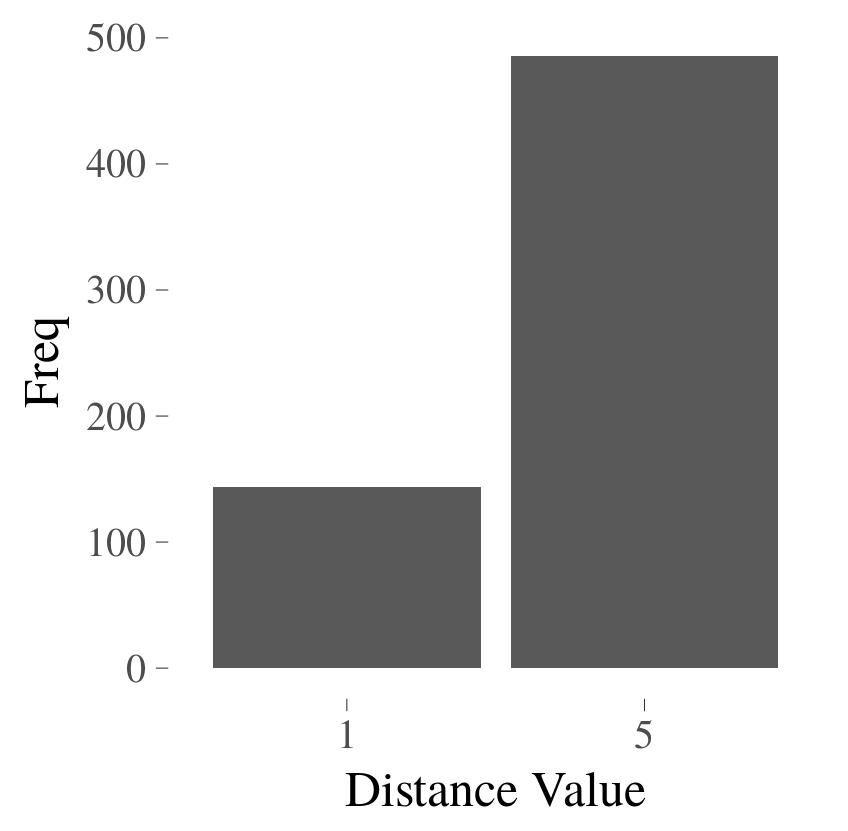}
\\
Histogram after two multiplications
&
Histogram after three multiplications
\end{tabular}
\caption{The process of distance equalization for successive powers of
the incidence matrix. The matrix $A_{D}^{3}$ is ultrametric.}\label{fig:latcc_4dc}
\end{figure*}

By shifting the data points to different locations, we create several
distinct structured clusterings that consists of
rectangular clusters.

Figures~\ref{fig:lat36alla} and~\ref{fig:lat36allb} show an example of
a series of datasets with a total of 36 data points.  Initially, the
data set has 4 rectangular clusters containing 9 data points each with
a gap of 3 distance units between the clusters.
The ultrametricity of the dataset and, therefore,  its clusterability
is affected by the number of
clusters, the size of the clusters, and the inter-cluster distances.
Figure~\ref{fig:lat36allb} shows that $m(A)$ reaches its highest
value  and, therefore, the clusterability is the lowest, when there is only one
cluster in the dataset (see the third row of
Figure~\ref{fig:lat36allb}).

\begin{figure*}
\centering
\begin{tabular}{ccc}
\includegraphics[width=0.27\textwidth]{latcc_4.jpeg}
&
\includegraphics[width=0.27\textwidth]{latcc_4_0.jpeg}
&
\includegraphics[width=0.27\textwidth]{latcc_4_3.jpeg}\\
Lattice with $k=4$ & Histogram for $k=4$ & $m(A_D) = 3, \clust(D) = 12$
\\
\includegraphics[width=0.27\textwidth]{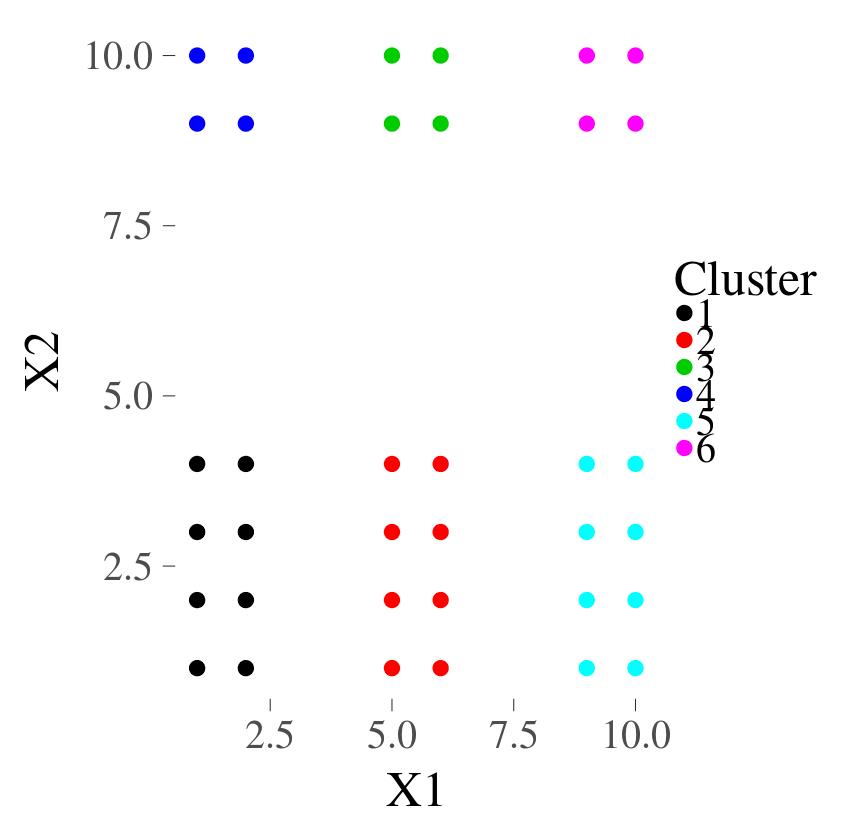}
&
\includegraphics[width=0.27\textwidth]{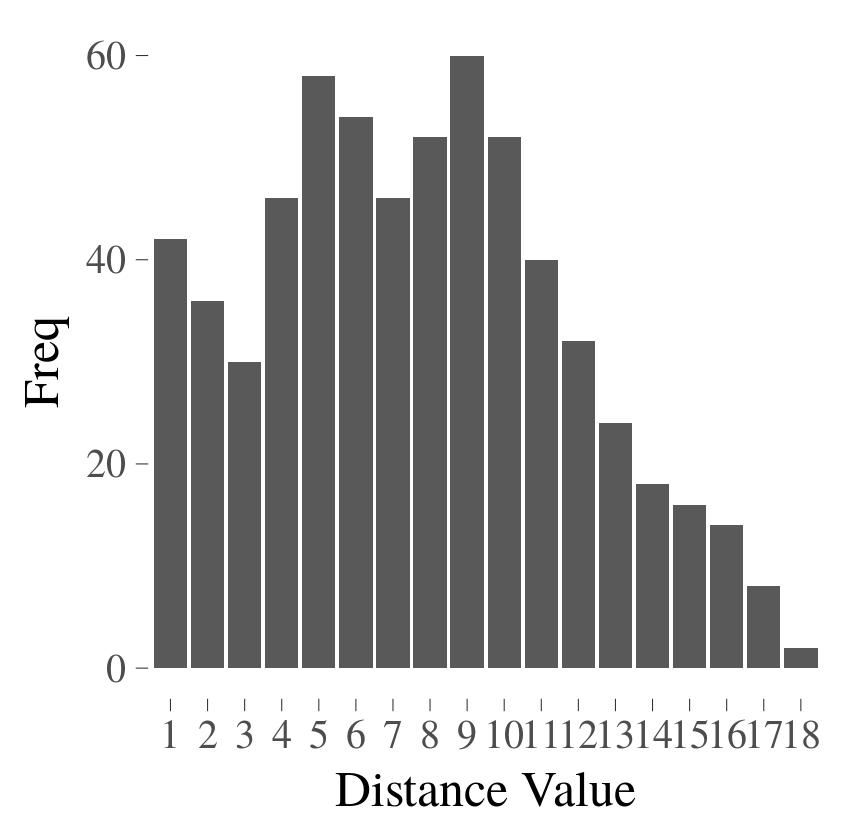}
&
\includegraphics[width=0.27\textwidth]{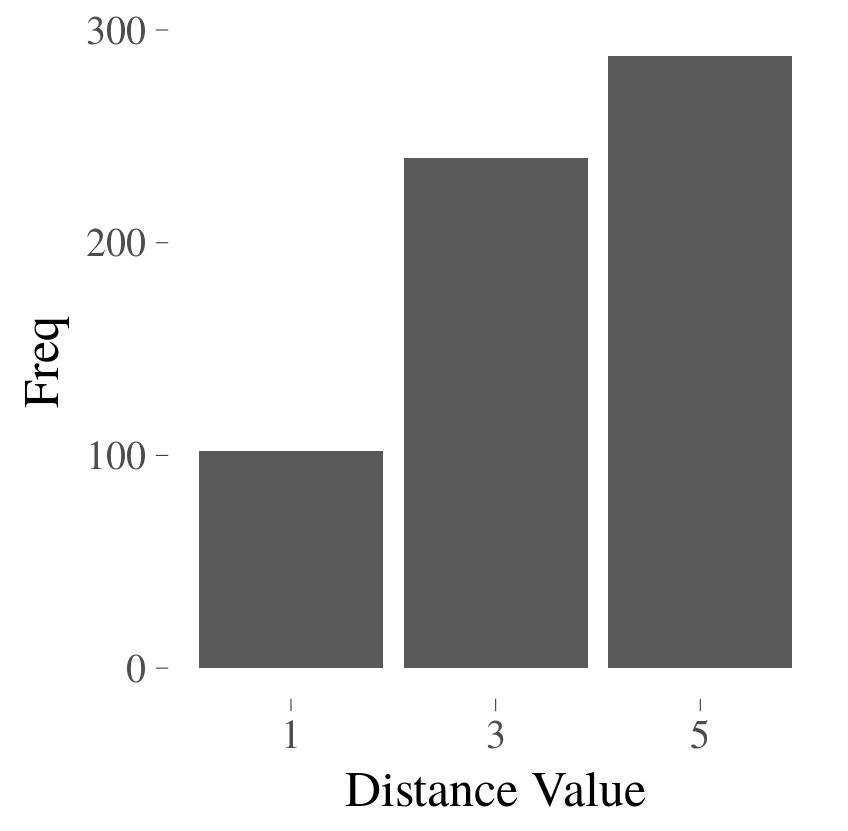}\\
Lattice with $k=6$ & Histogram for $k=6$ & $m(A_D) = 4, \clust(D) = 9$
\\
\includegraphics[width=0.27\textwidth]{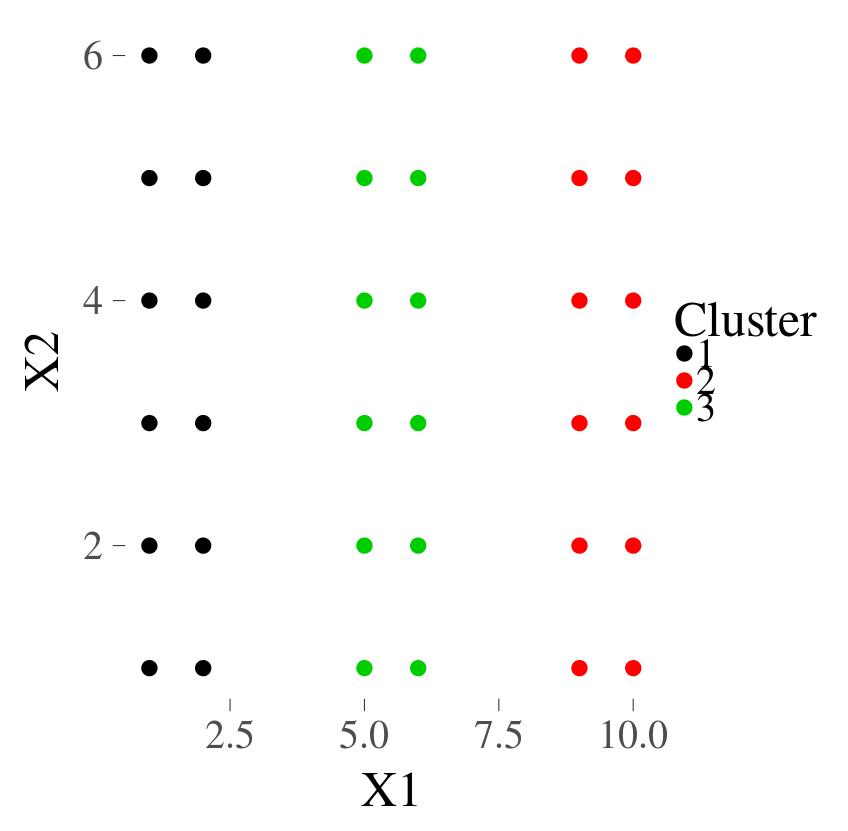}
&
\includegraphics[width=0.27\textwidth]{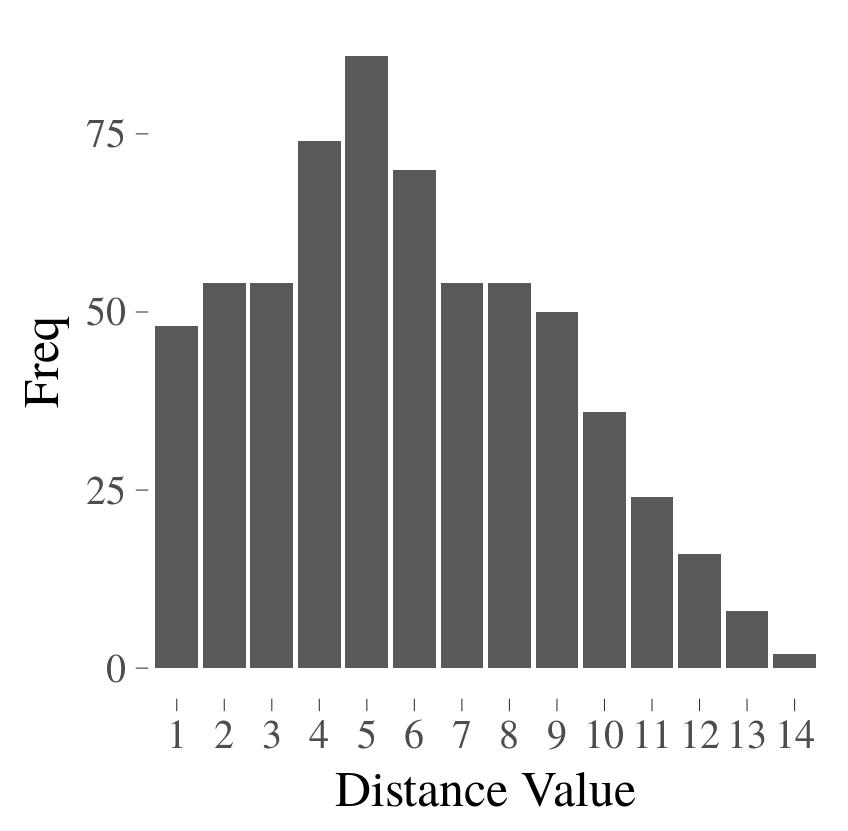}
&
\includegraphics[width=0.27\textwidth]{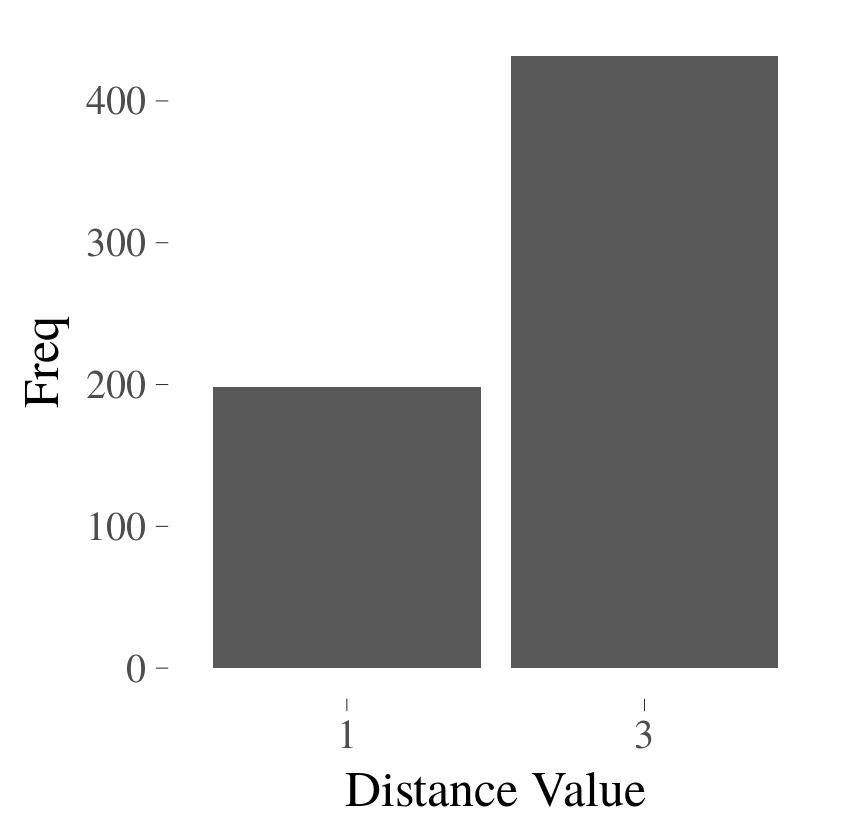}\\
Lattice with $k=3$ & Histogram for $k=3$ & $m(A_D) = 5, \clust(D) = 7.2$
\end{tabular}
\caption{Cluster separation and clusterability.\label{fig:lat36alla}}
\end{figure*}

If points are uniformly distributed, as it is the case in the
third row of Figure~\ref{fig:lat36allb}, the clustering structure
disappears and $\clust(D)$ has the lowest value.

\begin{figure*}
\centering
\begin{tabular}{ccc}
\includegraphics[width=0.27\textwidth]{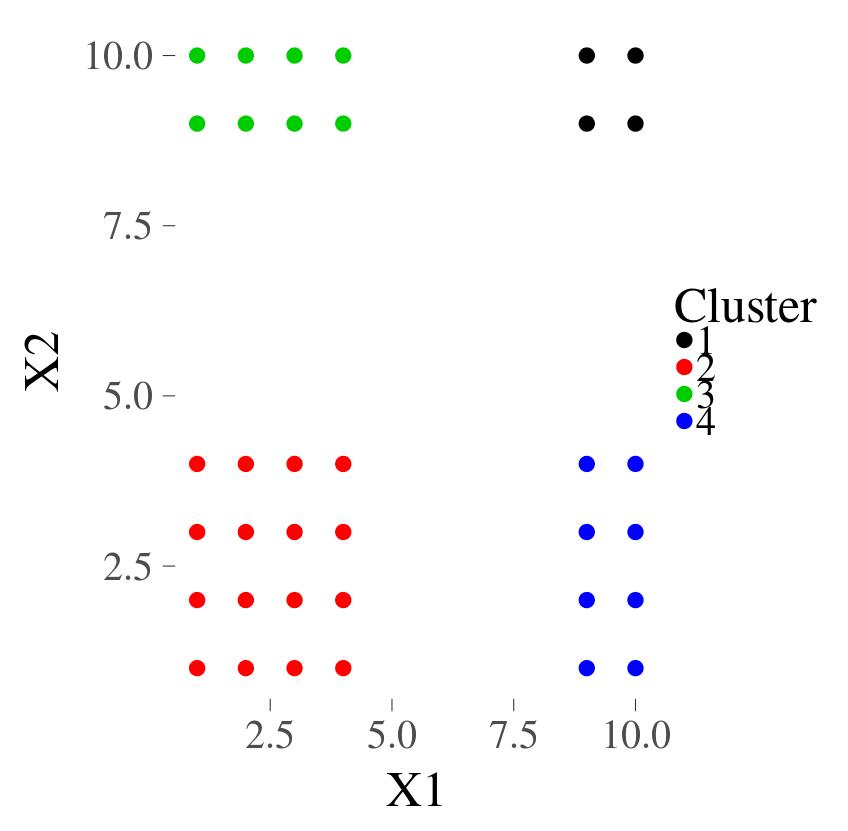}
&
\includegraphics[width=0.27\textwidth]{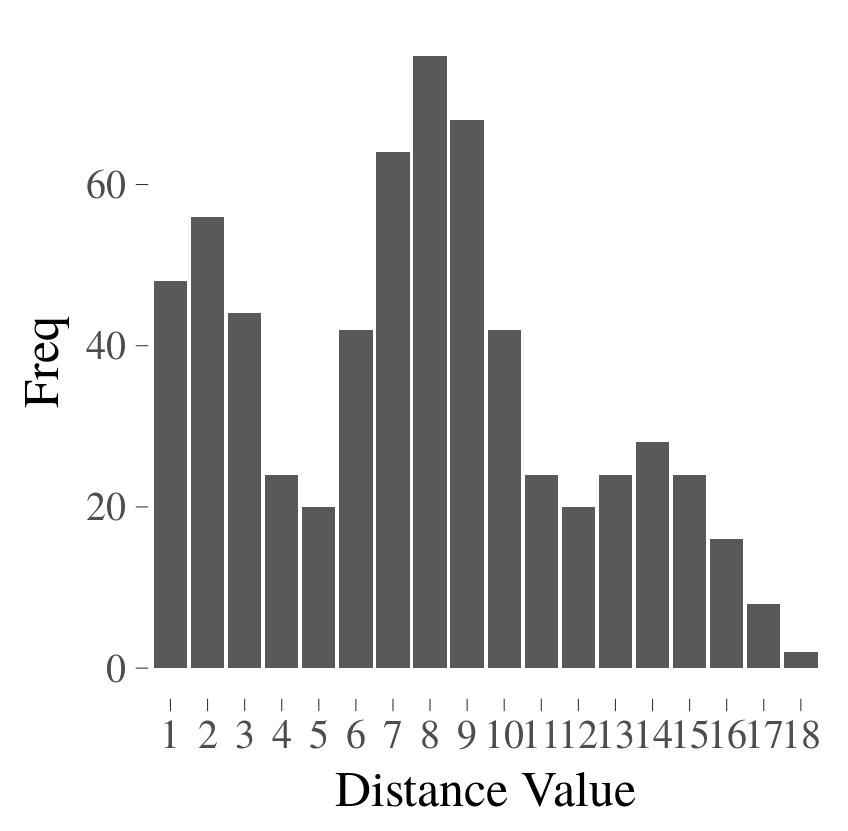}
&
\includegraphics[width=0.27\textwidth]{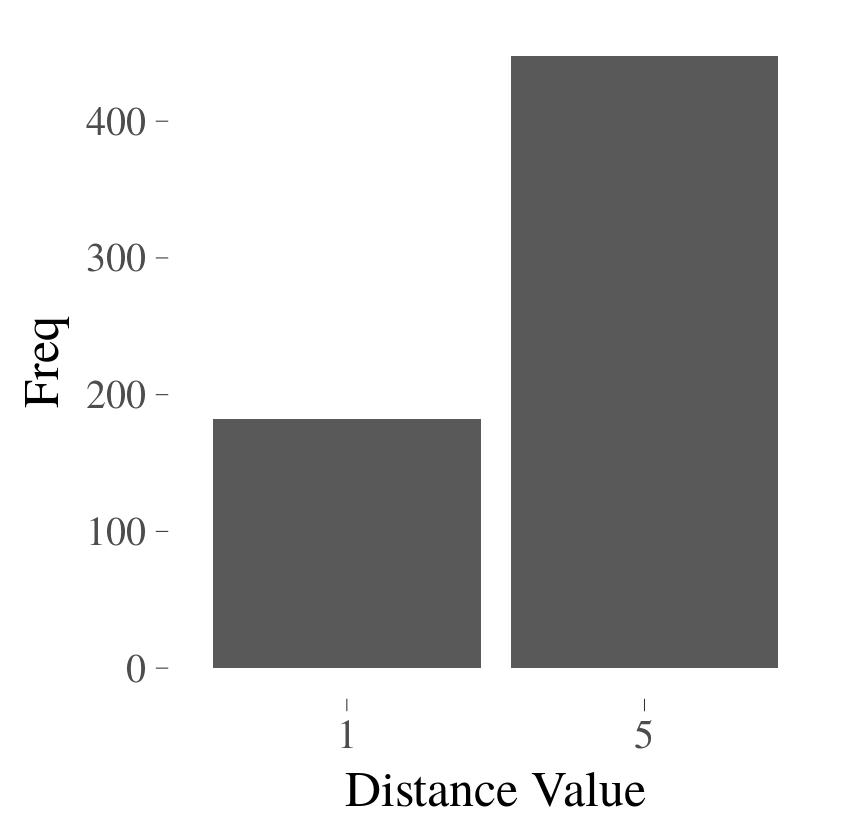}\\
\\
Lattice dataset with $k=4$
&
Histogram for $k=4$
&
$m(A_D) = 5, \clust(D) = 7.2$
\\
\includegraphics[width=0.27\textwidth]{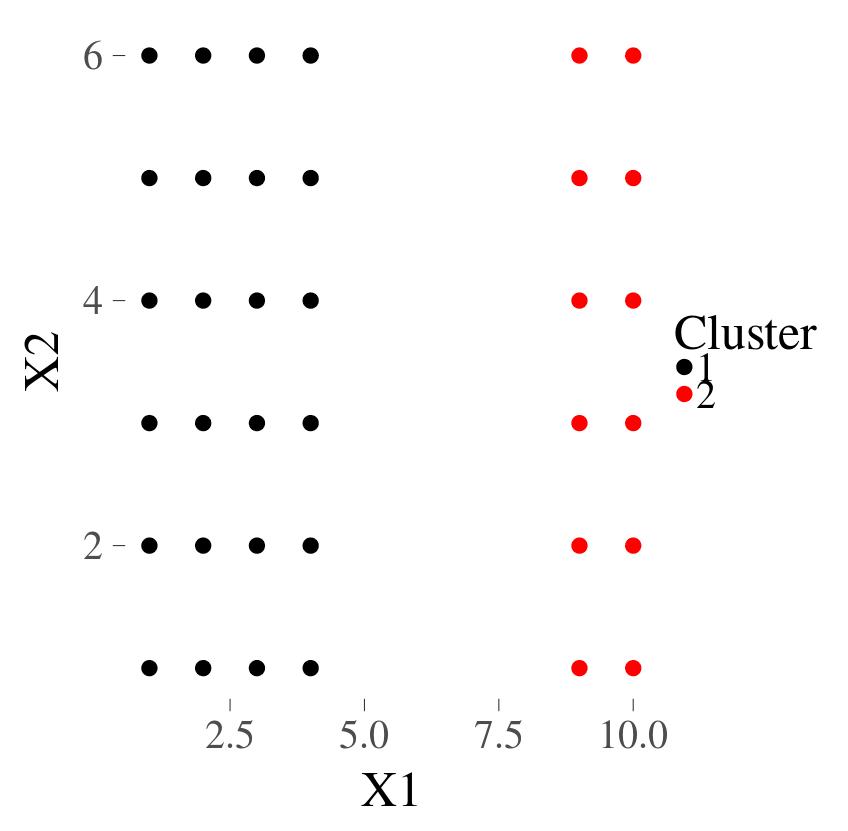}
&
\includegraphics[width=0.27\textwidth]{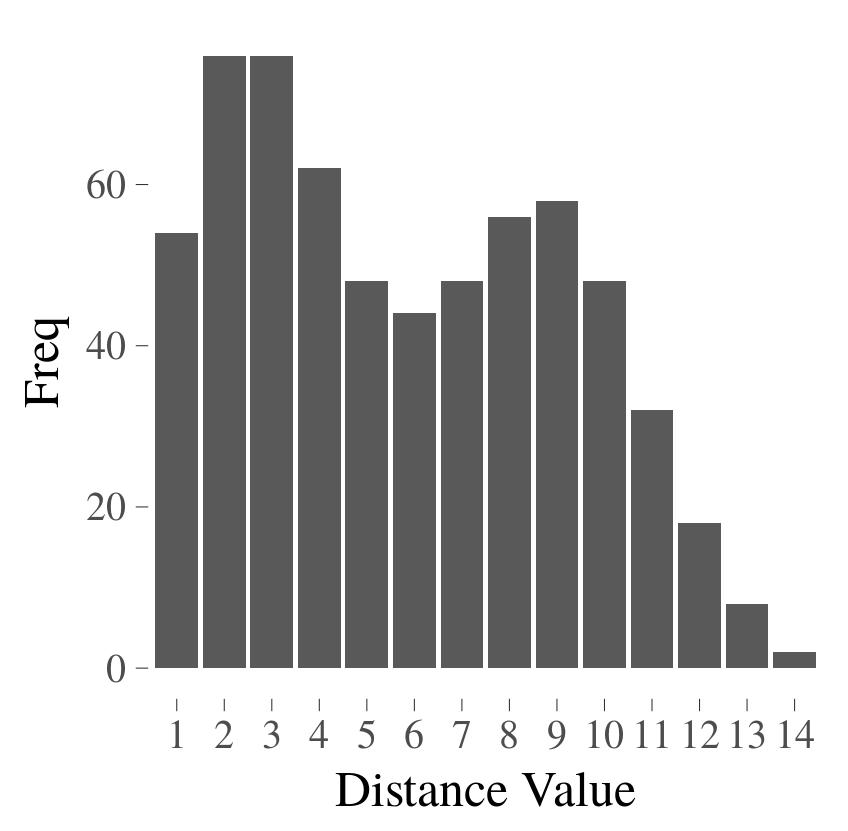}
&
\includegraphics[width=0.27\textwidth]{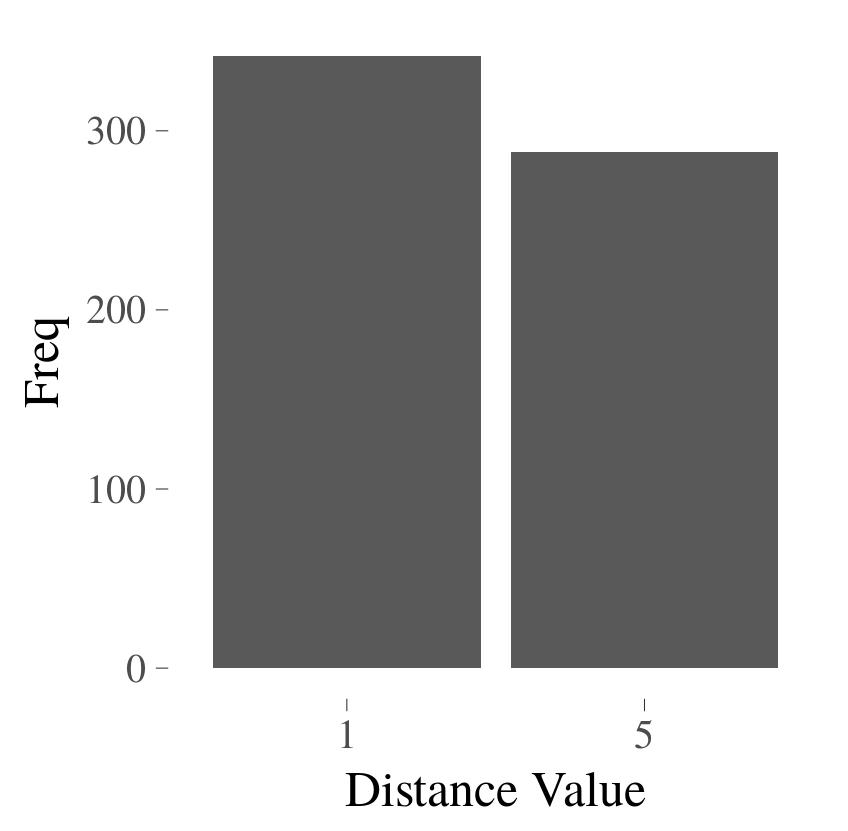}
\\
Lattice dataset with $k=2$
&
Histogram for $k=2$
&
$m(A_D) =7, \clust(D) = 5.1$
\\
\includegraphics[width=0.27\textwidth]{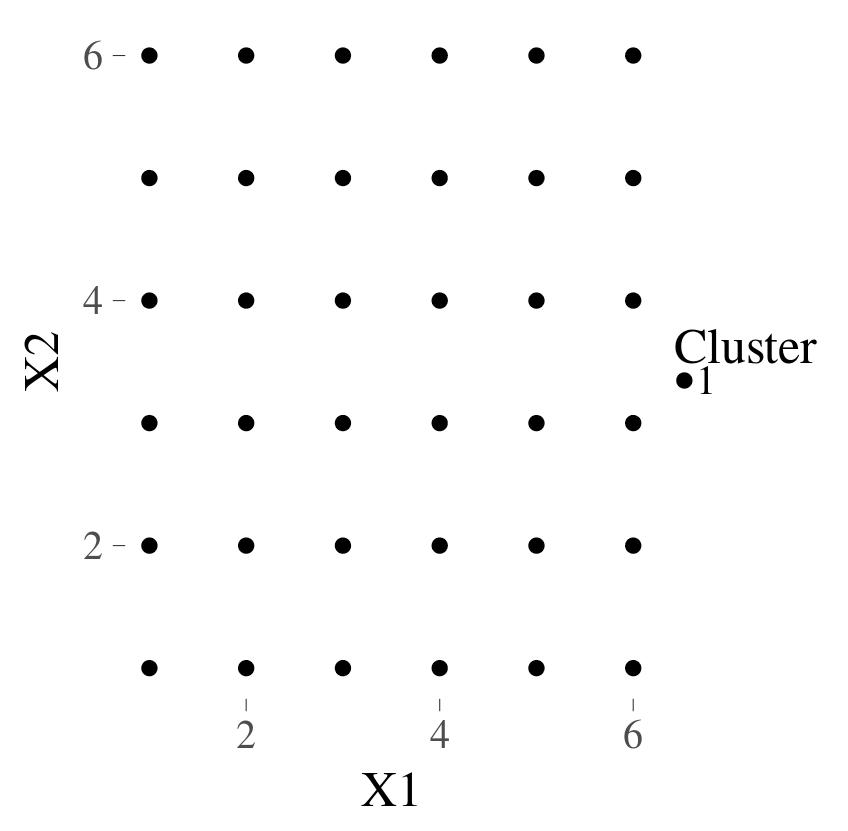}
&
\includegraphics[width=0.27\textwidth]{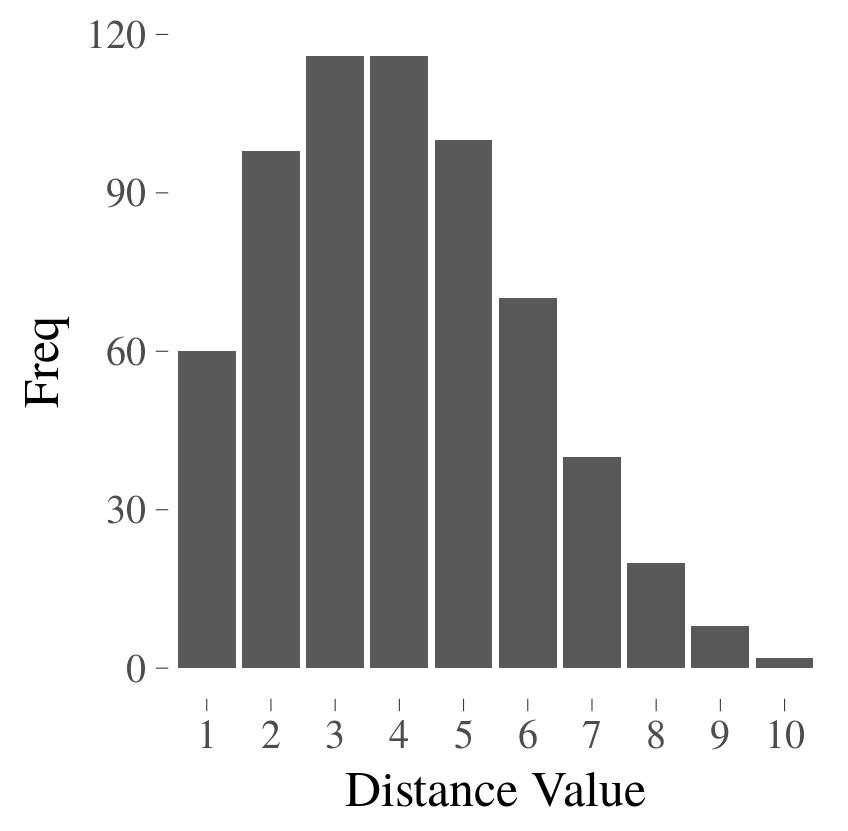}
&
\includegraphics[width=0.27\textwidth]{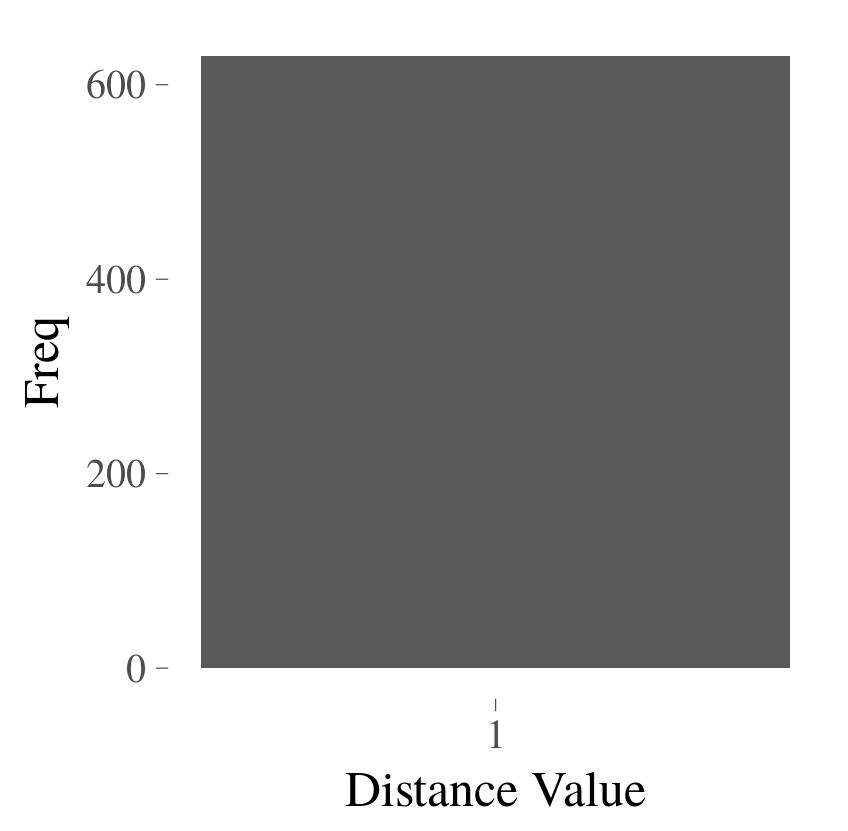}\\
Lattice dataset with $k=1$
&
Histogram for $k=1$
&
$m(A_D) =9, \clust(D) = 4$
\end{tabular}
\caption{Cluster separation and clusterability (continued).
\label{fig:lat36allb}}
\end{figure*}

\begin{figure*}
\centering
\begin{tabular}{ccc}
\includegraphics[width=0.27\textwidth]{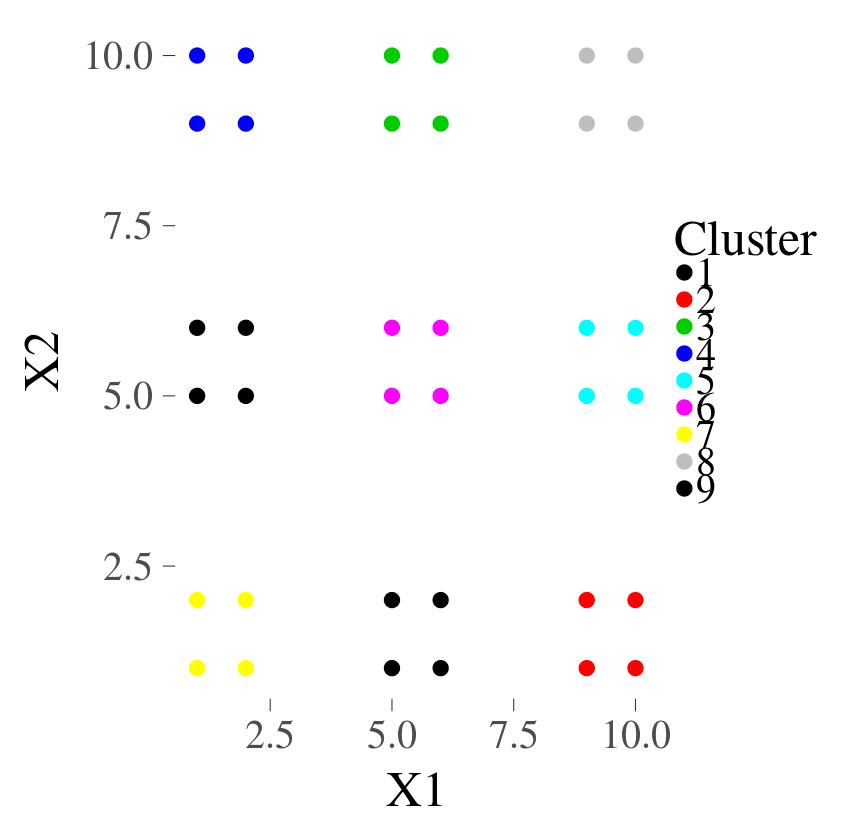}
&
\includegraphics[width=0.27\textwidth]{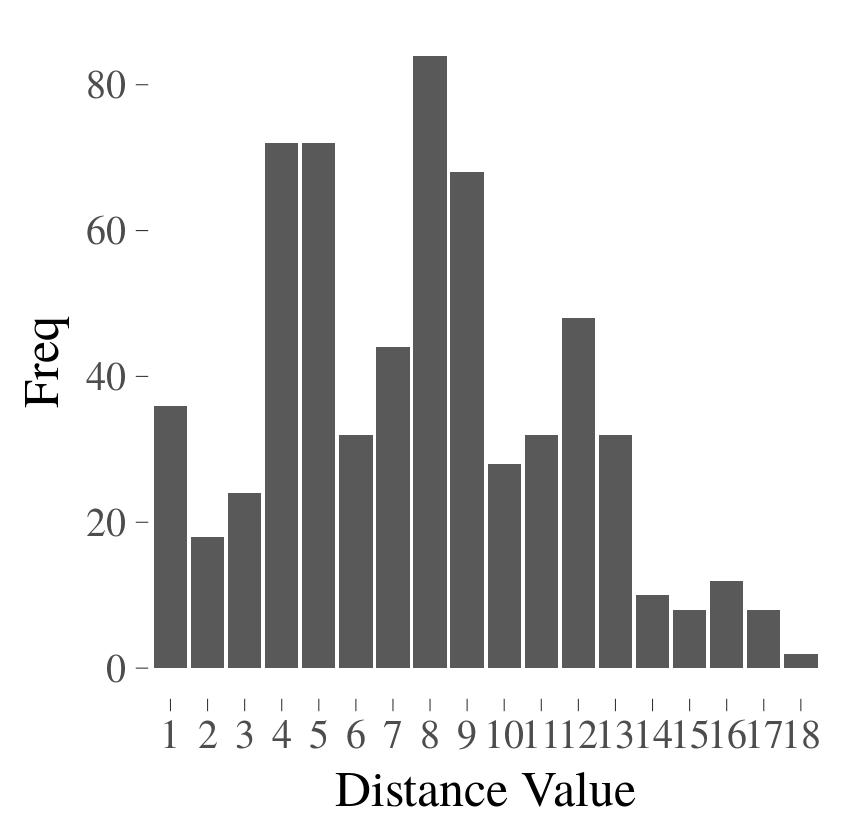}
&
\includegraphics[width=0.27\textwidth]{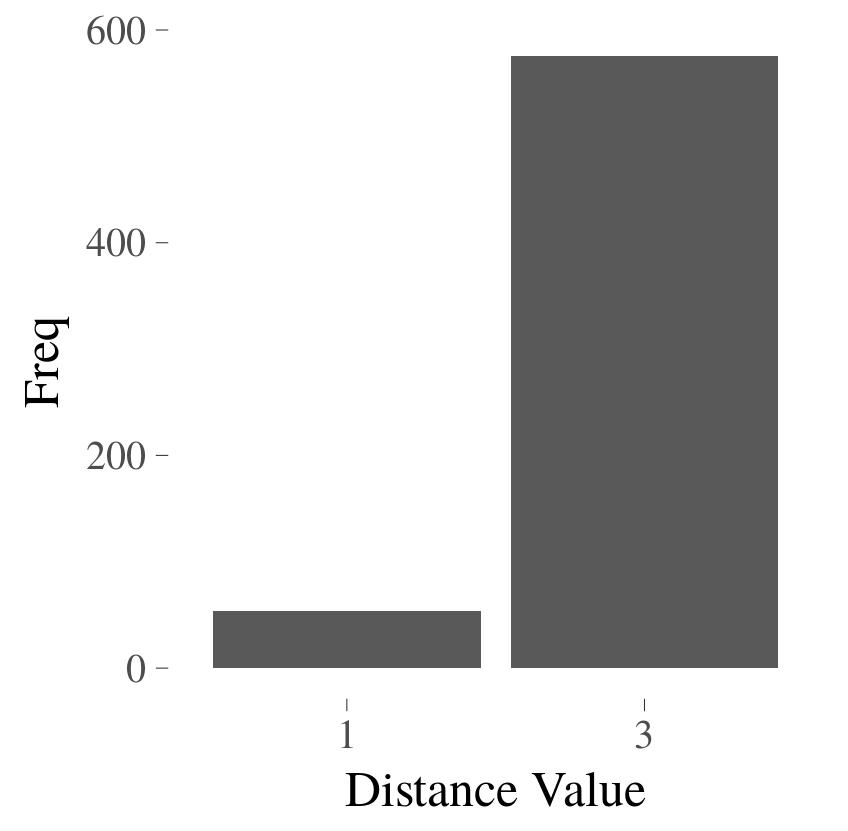}
\\
Lattice dataset with $k=9$
&
$k=9$
&
$m(A_D) =6, \clust(D)=6$
\end{tabular}
\caption{Further examples of data sets and their clusterability.}\label{fig:lat36spc}
\end{figure*}

Histograms are used by some authors~\cite{SBA,ASB3} to identify the
degree of clusterability.  Note however that in the case of the data
shown in Figures~\ref{fig:lat36alla} and~\ref{fig:lat36allb}, the
histograms of original dissimilarity of the dataset do not offer
guidance on the clusterability(second column of
Figure~\ref{fig:lat36alla} and~\ref{fig:lat36allb}).  By applying the
``min-max'' power operation on the original matrix, we get an
ultrametric matrix.  The new histogram of the ultrametric shows a
clear difference on each dataset. In the third column of
Figures~\ref{fig:lat36alla} and~\ref{fig:lat36allb}, the histogram of
the ultrametric matrix for each dataset shows a decrease of the number
of distinct distances after the ``power'' operation.

If the dataset has no clustering structure the histogram of the
ultrametric distance has only one bar.

The number of pics $p$ of the histogram indicate the minimum
number of clusters $k$ in the ultrametric space specified by the
matrix $A^*$ using the equality $\binom{k}{2} = p$, so the number
of clusters is $\left\lceil \frac{1 +
\sqrt{1+8p}}{2}\right\rceil$. The largest $k$ values of valleys of
the histogram indicate the radii of the spheres in the ultrametric
space that define the clusters.

If a data set contains a large number of small clusters, these
clusters can be regarded as outliers and the clusterability of the
data set is reduced.  This is the case in the third line of
Figure~\ref{fig:lat36spc} which shows a particular case for 9 clusters with
36 data points.  Since the size of each cluster is too small to be
considered as a real cluster, all of them together are merely regarded
as a one cluster dataset with 9 points.

\section{Conclusions and Future Work}\label{sec:cfw}
The special matrix powers of the adjacency matrix of the weighted
graph of object dissimilarities provide a tool for computing the
subdominant ultrametric of a dissimilarity and an assessment of the
existence of an underlying clustering structure in a dissimilarity
space.

The ``power'' operation successfully eliminates the redundant
information in the dissimilarity matrix of the dataset but maintains
the useful information that can discriminate the cluster structures of
the dataset.

In a series of seminal papers\cite{MURT04,MURT05,MURT07}, F. Murtagh
argued that as the dimensionality of a linear metric space increases,
an equalization process of distances takes place and the metric of the
space gets increasingly closer to an ultrametric.  This raises the
issues related to the comparative evaluation (statistical and
algebraic) of the ultrametricity of such spaces and of their
clusterability, which we intend to examine in the future.

\section*{References}
%\bibliographystyle{plain}
%\bibliography{db}

\end{document}